\documentclass[a4paper,11pt]{llncs}

\usepackage{times}
\usepackage{url}
\usepackage{latexsym}

\usepackage{color} 
\def\bR{\begin{color}{red}}  
\def\bB{\begin{color}{blue}}
\def\bM{\begin{color}{magenta}} 
\def\bC{\begin{color}{cyan}}
\def\bW{\begin{color}{white}} 
\def\bBl{\begin{color}{black}} 
\def\bG{\begin{color}{green}} 
\def\bY{\begin{color}{yellow}} 
\def\e{\end{color}}

\usepackage{amsmath}
\usepackage{amssymb}  
\usepackage{multirow}

\usepackage{tikz}
\usepackage{pgfplots}
\usetikzlibrary{trees}
\usetikzlibrary{topaths}
\usetikzlibrary{decorations.pathmorphing}
\usetikzlibrary{decorations.markings}
\usetikzlibrary{matrix,backgrounds,folding}
\usetikzlibrary{chains,scopes,positioning,fit}
\usetikzlibrary{arrows,shadows}
\usetikzlibrary{calc} 
\usetikzlibrary{chains}
\usetikzlibrary{shapes,shapes.geometric,shapes.misc}

\pgfdeclarelayer{edgelayer}
\pgfdeclarelayer{nodelayer}
\pgfsetlayers{background,edgelayer,nodelayer,main}

\tikzstyle{dot}=[circle,fill=black,draw=black]
\tikzstyle{every picture}=[baseline=(current bounding box).east,scale=0.5,node distance=5mm]
\tikzstyle{none}=[inner sep=0pt]
\tikzstyle{every loop}=[]

\tikzstyle{(null)}=[]
\tikzstyle{plain}=[]

\newcommand{\semantics}[1]{[\![ #1 ]\!]} 
\newcommand{\ov}{\overrightarrow}

\title{The Frobenius anatomy of word meanings I:\\
subject and object relative pronouns}

\author{}
\institute{}
\date{}

\begin{document}
\maketitle

\hspace{-3cm}\begin{tabular}{ccc}
Mehrnoosh Sadrzadeh & Stephen Clark &  Bob Coecke     \\
Queen Mary University of London &  University of Cambridge &  University of Oxford    \\
 School of Electronic Eng. and Computer Science &  Computer Laboratory & \quad Dept. of Computer Science    \\
{ \tt \small mehrs@eecs.qmul.ac.uk} &  {\tt \small sc609@cam.ac.uk} & {\tt \small coecke@cs.ox.ac.uk}   
\end{tabular}

\begin{abstract}

This paper develops a compositional vector-based semantics of
subject and object relative pronouns within a categorical
framework. Frobenius algebras are used to formalise the operations
required to model the semantics of relative pronouns, including
passing information between the relative clause and the modified noun phrase,
as well as copying,  combining, and discarding parts of the relative
clause. We develop two instantiations of the abstract semantics, one
based on a truth-theoretic approach and one based on corpus statistics.
\end{abstract}

\section{Introduction}

Ordered algebraic structures and sequent calculi have been used
extensively in Computer Science and Mathematical Logic. They have also
been used to formalise and reason about natural language.
 Lambek (1958) \cite{Lambek0} used the ordered algebra of residuated monoids to
model grammatical types, their juxtapositions and
reductions. Relational words such as verbs have implicative types and
are modelled using the residuals of the monoid multiplication \cite{Moortgat,Morrill}.
Later,  Lambek (1999) \cite{Lambek1} simplified these algebraic calculi in favour of
\emph{pregroups}. Here, there are no binary residual operations, but
each element of the algebra has a left and a right residual or adjoint.

In terms of semantics, pregroups do not naturally lend themselves to a
model-theoretic treatment \cite{Montague}.  However, pregroups are
suited to a radically different treatment of semantics, namely
distributional semantics \cite{Schutze}. Distributional semantics uses
vector spaces based on contextual co-occurrences to model the meanings
of words.   Coecke et al. (2010) \cite{Coeckeetal} show how a {\em compositional}
semantics can be developed within a vector-based framework, by
exploiting the fact that vector spaces with linear maps and pregroups
both have a compact closed categorical structure
\cite{KellyLaplaza,PrellerLambek}. Some initial attempts at
implementation include  Grefenstette and Sadrzadeh
(2011a) \cite{GrefenSadr1} and
 Grefenstette and Sadrzadeh
(2011b) \cite{GrefenSadr2}.

One problem with the distributional approach is that it is difficult
to see how the meanings of some words --- e.g. logical words such as
{\em and, or}, and relative pronouns such as {\em who, which, that,
  whose} --- can be modelled contextually.  Our focus in this paper is
on relative pronouns in the distributional compositional setting.

The difficulty with pronouns is that the contexts in which they occur
do not seem to provide a suitable representation of their meanings:
pronouns tend to occur with a great many nouns and verbs. Hence, if
one applies the contextual co-occurrence methods of distributional
semantics to them, the result will be a set of dense vectors which do
not discriminate between different meanings.
The current
state-of-the-art in compositional distributional semantics either
adopts a simple method to obtain a vector for a sequence of words,
such as adding or mutliplying the contextual vectors of the words
\cite{Mitchell:Lapata:08}, or, based on the grammatical structure,
builds linear maps for some words and applies these to the vector
representations of the other words in the string
\cite{BaroniEMNLP10,GrefenSadr1}. Neither of these approaches produce
vectors which provide a good representation for the meanings of
relative clauses.

In the grammar-based approach, one has to assign a linear map to the
relative pronoun, for instance a map $f$ as follows:
\[
\ov{\text{men who like Mary}} = f(\ov{\text{men}}, \ov{\mbox{like Mary}})
\]
However, it is not clear what this map should be.  Ideally, we do not
want it to depend on the frequency of the co-occurrence of the
relative pronoun with the relevant basis vectors. But both of the above
mentioned approaches rely heavily on the information provided by a
corpus to build their linear maps.  The work of
 Baroni and Zamparelli
(2010) \cite{BaroniEMNLP10} uses linear regression and approximates the
context vectors of phrases in which the target word has occurred,
and the work of  Grefenstette
and Sadrzadeh (2011a) \cite{GrefenSadr1} uses the sum of Kronecker
products of the arguments of the target word across the corpus. 

The semantics we develop for relative pronouns and clauses uses the
general operations of a Frobenius algebra over vector spaces
\cite{CoeckeVic} and the structural categorical morphisms of vector
spaces. We do not rely on the co-occurrence frequencies of the
pronouns in a corpus and only take into account the structural roles
of the pronouns in the meaning of the clauses.  The computations of
the algebra and vector spaces are depicted using string diagrams
\cite{Joyal}, which depict the interactions that occur among the words
of a sentence.  In the particular case of relative clauses, they
visualise the role played by the relative pronoun in passing information
between the clause and the modified noun phrase, as well as copying,
combining, and even discarding parts of the relative clause.

We develop two instantiations of the abstract semantics, one based on
a truth-theoretic approach, and one based on corpus statistics, where
for the latter the categorical operations are instantiated as matrix multiplication
and vector component-wise multiplication.  As a result, we will obtain the following for the meaning of a subjective relative clause:
\[
\ov{\text{men who like Mary}} = \ov{\text{men}} \odot (\overline{\text{love}}  \times \ov{\text{Mary}})
\]
The rest of the paper
introduces the categorical framework, including the formal definitions
relevant to the use of Frobenius algebras, and then shows how these
structures can be used to model relative pronouns within the
compositional vector-based setting.

\section{Compact Closed Categories and  Frobenius Algebras}
\label{sec:frob}

We briefly review the main concepts of compact closed categories and Frobenius Algebras, focusing on the concepts and constructions that we are going to use in our models.  For a formal presentation, see \cite{KellyLaplaza,Kock,Baez}; for an informal introduction, see \cite{CoeckePaq}. 

A compact closed category has objects $A, B, C, ...$, morphisms $f \colon A \to B$, $g \colon B \to C$, composition $g\circ f$, a monoidal tensor  $A \otimes B$ that has a unit $I$, and  for each object $A$  two  objects $A^r$ and $A^l$ together with the following morphisms:
\begin{align*}
A \otimes A^r   \stackrel{\epsilon_A^r} {\longrightarrow} &I   \stackrel{\eta_A^r}{\longrightarrow} A^r \otimes A\\
A^l \otimes A \stackrel{\epsilon_A^l}{\longrightarrow} &I \stackrel{\eta_A^l}{\longrightarrow} A \otimes A^l\
\end{align*}
These  structural morphisms  of the compact  closed structure  satisfy the following equalities, some times referred to as \emph{yanking}, where $1_A$ is the identity morphism on the object $A$:
\begin{align*}
(1_A \otimes \epsilon_A^l) \circ (\eta_A^l \otimes 1_A) & = 1_A \hspace{2cm} (\epsilon_A^r \otimes 1_A) \circ (1_A \otimes \eta_A^r) = 1_A\\
(\epsilon_A^l \otimes 1_A) \circ (1_{A^l} \otimes \eta_A^l) &= 1_{A^l} \hspace{2cm} (1_{A^r} \otimes \epsilon_A^r) \circ (\eta_A^r \otimes 1_{A^r}) = 1_{A^r}
\end{align*}

\medskip
A pregroup is a  partially  ordered  compact closed category. This means that the  objects   are elements of a partially ordered monoid and between any two objects $p,q$ there exists a morphism of type $p \to q$ iff $p \leq q$. Compositions of morphisms are obtained by transitivity and the identities by  reflexivity of the partial order. The  tensor of the category is the monoid multiplication, and the four yanking inequalities become as follows:
\begin{align*}
\epsilon_p^r = p \cdot p^r \leq 1 &\hspace{2cm} \epsilon_p^l = p^l \cdot p \leq 1\\
\eta_p^r = 1 \leq p^r \cdot p &\hspace{2cm}  \eta_p^l = 1 \leq p \cdot p^l
\end{align*}
In the rest of the paper, by  $Preg$ we either mean any or a particular pregroup, as determined by  the context. 

Finite dimensional vector spaces and linear maps also constitute a compact closed category, to which we refer as  $FVect$.  Finite dimensional vector spaces $V, W, ...$ are  objects of this category;  linear maps $f \colon V \to W$ are its morphisms with composition being  the composition of linear maps. The  tensor product of the spaces $V \otimes W$ is the  linear algebraic tensor product, whose  unit is the scalar field of the vector spaces; in our case this is the field of reals $\mathbb{R}$.  As opposed to the tensor product in $Preg$,  the tensor between vector spaces is symmetric, hence we have a natural isomorphism $V \otimes W \cong  W \otimes V$. As a result of the symmetry of the tensor,  
we   have   $V^l \cong V^r \cong V^*$, where $V^*$ is  the dual  space  of  $V$. When the basis vectors of  vector spaces are fixed, we furthermore obtain that $V^* \cong V$. Elements of  a vector space $V$, that is   vectors $\overrightarrow{v} \in V$, are presented by linear maps of type $\mathbb{R} \to V$, which can be  defined by setting $1\mapsto \overrightarrow{v}$ when relying on linearity.  We still denote these maps by $\overrightarrow{v}$. 

Given a basis $\{r_i\}_i$ for a vector spaces $V$, the epsilon maps are given by the  inner product extended by linearity, that is we have:
\[
\epsilon^l  =  \epsilon^r \colon   V^* \otimes V \to \mathbb{R} \quad \mbox{given by} \quad  \sum_{ij} c_{ij} \ \psi_i \otimes \phi_j  \quad \mapsto \quad \sum_{ij} c_{ij} \langle \psi_i \mid \phi_j \rangle
\]
Eta maps are   defined as  the linear map representing a vector in $V \otimes V^*$:
\[
\eta^l = \eta^r \colon   \mathbb{R} \to V \otimes V^* \quad \mbox{given by} \quad    1 \quad \mapsto \quad \sum_i r_i \otimes r_i
\]
As  mentioned above,  the  assignment of an image to $1\in \mathbb{R}$  extends to all  numbers by linearity.

A Frobenius algebra over a symmetric monoidal  category $({\cal C}, \otimes, I)$ is a tuple $(X, \Delta, \iota, \mu, \zeta)$ where for $X$ an object of ${\cal C}$,  the triple
\[
(X, \Delta\colon X \to X \otimes X, \iota\colon X \to I)
\]
 is an internal comonoid,  the triple 
\[
(X, \mu\colon  X \otimes X \to X, \zeta\colon I \to X)
\]
 is an  internal monoid, and which together moreover  satisfy the following \emph{Frobenius} condition:

\vspace{-0.4cm}
\begin{align*}
(\mu \otimes 1_X) \circ (1_X \otimes \Delta) \ = \  \Delta \circ \mu  \ = \  (1_X \otimes \mu) \circ (\Delta \otimes 1_X) 
\end{align*}

Frobenius algebras were originally introduced in 1903 by F.~G.~Frobenius in the context of proving representation theorems for group theory \cite{Frob}. Since then,  they have found applications in other fields of mathematics and physics, e.g. in topological quantum field theory \cite{Kock}. The above general categorical definition is due to Carboni and Walters \cite{CarboniWal}. 

In the category of finite dimensional vector spaces and linear maps \textbf{FVect},  any vector space $V$ with a fixed basis $\{\ov{v_i}\}_i$ has a Frobenius algebra over it, explicitly given by:
\[
\begin{array}{lll}
\Delta :: \ov{v_i} \mapsto \ov{v_i} \otimes \ov{v_i} &\qquad\qquad& \iota:: \ov{v_i} \mapsto 1\vspace{3mm}\\
\mu:: \ov{v_i} \otimes \ov{v_j} \mapsto 
\delta_{ij}\ov{v_i} :=\left\{ \begin{array}{lll}
\ov{v_i} &\ \ \  & i=j\\
\ov{0} & & i\not = j
\end{array}\right.
&\qquad\qquad& \zeta:: 1 \mapsto   \sum_i  \ov{v_i}  
\end{array}
\]

Frobenius algebras  over  vector spaces with orthonormal bases  are moreover  \emph{special} and \emph{commutative}. A commutative  Frobenius algebra satisfies the following two conditions for $\sigma: X \otimes Y \to Y \otimes X$ the symmetry morphism of $({\cal C}, \otimes, I)$:
\[
\sigma \circ \Delta = \Delta \hspace{2cm} \mu \circ \sigma = \mu
\]
A special Frobenius algebra is one that satisfies the following axiom:
\[
\mu \circ \Delta = 1
\]
The vector spaces of distributional models that we work with have fixed  orthonormal bases, hence they have special commutative Frobenius algebras over them. 

The comonoid's comultiplication  of a special commutative Frobenius algebra over a vector space    encodes  vectors  of lower dimensions  into vectors  of  higher dimensional tensor spaces; this operation is usually referred to by \emph{copying}. In linear algebraic terms, $\Delta(\ov{v})\in  V \otimes V$ is a diagonal matrix whose diagonal elements are the coefficients (or weights)  of $\ov{v} \in V$.  The corresponding algebraic operations encode elements of  higher dimensional tensor spaces into lower dimensional spaces; this operation is usually referred to by  \emph{uncopying}. For $\ov{w} \in V\otimes V$, when represented as a matrix, we have that $\mu(\ov{w}) \in V$ is a vector consisting only of the diagonal elements of $\ov{w}$.

As a concrete example, take $V$ to be a two dimensional space with basis $\{\ov{v_1}, \ov{v_2}\}$, then the basis of $V \otimes V$ is $\{\ov{v_1} \otimes \ov{v_1}, \ov{v_1} \otimes \ov{v_2}, \ov{v_2} \otimes \ov{v_1}, \ov{v_2} \otimes \ov{v_2}\}$. For a vector $v = a \ov{v_1} + b \ov{v_2}$ in $V$ we have:
\[
\Delta (v) = \Delta \left(\begin{array}{c} a\\b\end{array}\right)
 = 
 \left(\begin{array}{cc} a&0\\0&b\end{array}\right) = a \, \ov{v_1} \otimes \ov{v_1} + b \, \ov{v_2} \otimes \ov{v_2} 
\]
And for  $\ov{w} = a \, \ov{v_1} \otimes \ov{v_1} + b \, \ov{v_1} \otimes \ov{v_2} + c \, \ov{v_2} \otimes \ov{v_1} + d \, \ov{v_2} \otimes \ov{v_2}$ in $V \otimes V$,  we have:
\[
\mu (w) = \mu \left(\begin{array}{cc} a&b\\c&d\end{array}\right) = 
\left(\begin{array}{c} a\\d\end{array}\right)
= a \, \ov{v_1} + d \, \ov{v_2}
\]

\section{String Diagrams} 

The  framework of compact closed categories and Frobenius algebras comes with a  diagrammatic calculus that visualises  the derivations thereof. These diagrams also simplify the categorical and vector space computations of our setting to a great extent. 

We briefly introduce the fragment of this calculus that we will use. Morphisms are depicted by boxes and objects   by lines,  representing their identity morphisms. For instance  a morphism $f \colon A \to B$, and an object $A$ with the identity arrow  $1_A \colon A \to A$ are depicted as follows:

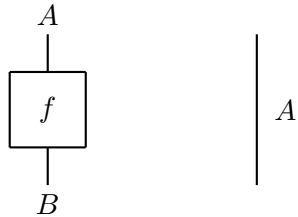
\vspace{-0.2cm}
\begin{center}
  {%
\beginpgfgraphicnamed{compact-diag}
\InputIfFileExists{compact-diag.tikz}{}{\input{./tikz/compact-diag.tikz}}
\endpgfgraphicnamed}
\end{center}

\noindent The tensor products of the objects and morphisms are depicted by juxtaposing their diagrams side by side, whereas compositions of morphisms are depicted by putting one on top of the other, for instance the object $A \otimes B$, and the morphisms $f \otimes g$ and $f \circ h$, for $f \colon A \to B, g \colon C \to D$, and $h \colon B \to C$ are depicted as follows:

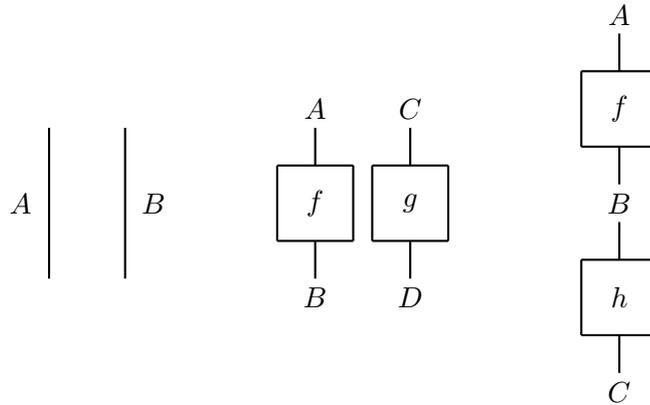
\vspace{-0.5cm}
\begin{center}
  {%
\beginpgfgraphicnamed{compact-diag-tensor}
\InputIfFileExists{compact-diag-tensor.tikz}{}{\input{./tikz/compact-diag-tensor.tikz}}
\endpgfgraphicnamed}
\end{center}

\vspace{-0.2cm}
\noindent The $\epsilon$ maps are depicted by cups,  $\eta$ maps by caps, and yanking by their composition and straightening the strings.  For instance, the diagrams for  $\epsilon^l \colon A^l \otimes A \to I$,  $\eta \colon I \to A\otimes A^l$ and $(\epsilon^l \otimes 1_A) \circ (1_A \otimes \eta^l)  = 1_A$ are  as follows:
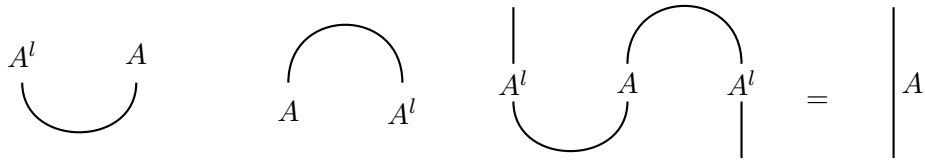
\begin{center}
  {%
\beginpgfgraphicnamed{compact-cap-cup}
\InputIfFileExists{compact-cap-cup.tikz}{}{\input{./tikz/compact-cap-cup.tikz}}
\endpgfgraphicnamed}
  \qquad
    {%
\beginpgfgraphicnamed{compact-yank}
\InputIfFileExists{compact-yank.tikz}{}{\input{./tikz/compact-yank.tikz}}
\endpgfgraphicnamed}
\end{center}

\noindent The composition of the  $\epsilon$ and $\eta$ maps with other morphisms is depicted as before, that is by juxtaposing them one above the other. For instance the compositions $(1_{B^l} \otimes f)\circ \epsilon^l$ and $\eta^l \circ (1_{A^l} \otimes f)$ are depicted as follows:
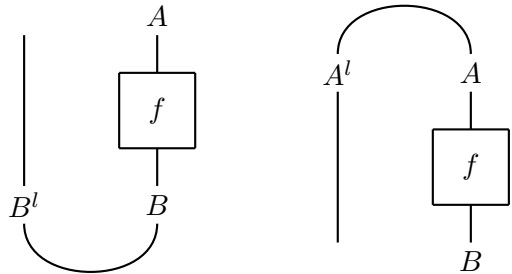
\begin{center}
  {%
\beginpgfgraphicnamed{compact-comp-mix}
\InputIfFileExists{compact-comp-mix.tikz}{}{\input{./tikz/compact-comp-mix.tikz}}
\endpgfgraphicnamed}
\end{center}
 
As for Frobenius algebras, the   comonoid and monoid morphisms  are depicted as follows:

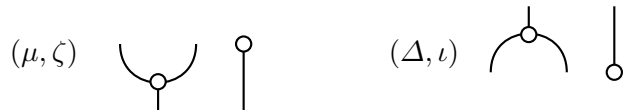
\begin{center}
{%
\beginpgfgraphicnamed{comp-alg-coalg}
\InputIfFileExists{comp-alg-coalg.tikz}{}{\input{./tikz/comp-alg-coalg.tikz}}
\endpgfgraphicnamed}
\end{center}

\noindent The \emph{Frobenius} condition is depicted as follows:

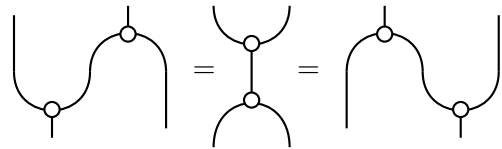
\begin{center}
{%
\beginpgfgraphicnamed{equation}
\InputIfFileExists{equation.tikz}{}{\input{./tikz/equation.tikz}}
\endpgfgraphicnamed}
\end{center} 

%
%
%
%
\noindent The defining equations of a commutative special Frobenius algebra guarantee that any picture depicting  a Frobnius  computation can be reduced to a normal form that only depends on the number of input and output strings of the nodes, independently of their topology.  These normal forms can be simplified to so-called `spiders': 
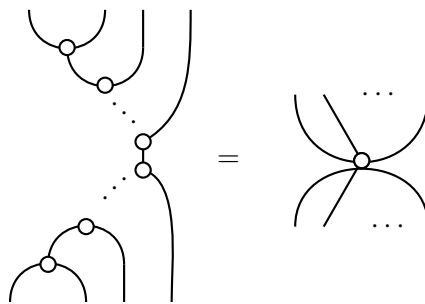
\[
{%
\beginpgfgraphicnamed{spider111}
\begin{tikzpicture}
	\begin{pgfonlayer}{nodelayer}
		\node [style=none] (0) at (-4.75, 3.25) {};
		\node [style=none] (1) at (-2.75, 3.25) {};
		\node [draw, thick, style=none, minimum size=0.2 cm, circle, fill=white] (2) at (-3.75, 2.25) {};
		\node [style=none] (3) at (-1.75, 2.25) {};
		\node [style=none] (4) at (-1.75, 3.25) {};
		\node [draw, thick, style=none, minimum size=0.2 cm, circle, fill=white] (5) at (-2.75, 1.25) {};
		\node [style=none] (6) at (-1.75, -0.25) {};
		\node [style=none] (7) at (-0.5, 3.25) {};
		\node [style=none] (8) at (0.5, -0.75) {$=$};
		\node [style=none] (9) at (2.25, 1) {};
		\node [style=none] (10) at (5.75, 1) {};
		\node [draw, thick, style=none, minimum size=0.2 cm, circle, fill=white] (11) at (4, -0.75) {};
		\node [style=none] (12) at (3, 1) {};
		\node [style=none] (13) at (4.5, 1) {$\cdots$};
		\node [style=none] (14) at (2.25, -2.5) {};
		\node [style=none] (15) at (5.75, -2.5) {};
		\node [style=none] (16) at (4.75, -2.5) {$\cdots$};
		\node [style=none] (17) at (-5.25, -4.5) {};
		\node [style=none] (18) at (-3.25, -4.5) {};
		\node [draw, thick, style=none, minimum size=0.2 cm, circle, fill=white] (19) at (-4.25, -3.5) {};
		\node [style=none] (20) at (-3.25, -2.5) {};
		\node [style=none] (21) at (-1, -4.5) {};
		\node [style=none] (22) at (-1.75, -1) {};
		\node [style=none] (23) at (-2.25, -3.5) {};
		\node [style=none] (24) at (-2.25, -4.5) {};
		\node [draw, thick, style=none, minimum size=0.2 cm, circle, fill=white] (25) at (-3.25, -2.5) {};
		\node [draw, thick, style=none, minimum size=0.2 cm, circle, fill=white] (26) at (4, -0.75) {};
		\node [style=none] (27) at (3, -2.5) {};
		\node [draw, thick, style=none, minimum size=0.2 cm, circle, fill=white] (28) at (-1.75, -0.25) {};
		\node [draw, thick, style=none, minimum size=0.2 cm, circle, fill=white] (29) at (-1.75, -1) {};
		\node [style=none] (30) at (-2.75, -1.75) {$\cdot$};
		\node [style=none] (31) at (-2.5, -1.5) {$\cdot$};
		\node [style=none] (32) at (-2.25, -1.25) {$\cdot$};
		\node [style=none] (33) at (-2.5, 0.75) {$\cdot$};
		\node [style=none] (34) at (-2.25, 0.5) {$\cdot$};
		\node [style=none] (35) at (-2, 0.25) {$\cdot$};
	\end{pgfonlayer}
	\begin{pgfonlayer}{edgelayer}
		\draw [style=thick, bend right=90, looseness=1.75] (0.center) to (1.center);
		\draw [style=thick, bend right=90, looseness=1.75] (2.center) to (3.center);
		\draw [style=thick] (4.center) to (3.center);
		\draw [style=thick, bend left=270, looseness=1.75] (9.center) to (10.center);
		\draw [style=thick] (12.center) to (11.center);
		\draw [style=thick, bend left=90, looseness=1.50] (14.center) to (15.center);
		\draw [style=thick, in=90, out=-15, looseness=0.75] (22.center) to (21.center);
		\draw [style=thick, in=30, out=-90] (7.center) to (6.center);
		\draw [thick, bend left=90, looseness=1.75] (17.center) to (18.center);
		\draw [thick, bend left=90, looseness=1.75] (19.center) to (23.center);
		\draw [style=thick](23.center) to (24.center);
		\draw [style=thick](27.center) to (26.center);
		\draw [style=thick](6.center) to (22.center);
	\end{pgfonlayer}
\end{tikzpicture}}
\endpgfgraphicnamed}
\]

 In the category $FVect$, apart from  spaces $V, W$, which are objects of the category, we also have  vectors $\ov{v}, \ov{w}$. These are depicted by  their representing morphisms and as   triangles with a number of strings emanating  from them. The number of strings of a triangle denote the tensor order of the  vector, for instance,  $\ov{v} \in V, \ov{v'} \in V \otimes W$, and $\ov{v''} \in V \otimes W \otimes Z$ are depicted as follows:
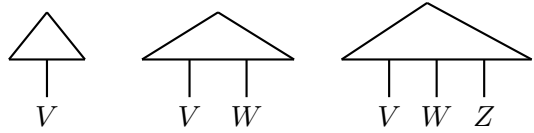
\begin{center}
  {%
\beginpgfgraphicnamed{compact-diag-triangle}
\InputIfFileExists{compact-diag-triangle.tikz}{}{\input{./tikz/compact-diag-triangle.tikz}}
\endpgfgraphicnamed}
\end{center}
Application of an $\epsilon$ map to a vector is depicted similarly, for instance $\epsilon^l(\ov{v})$ is depicted by the  composition $V^l \otimes V \stackrel{\epsilon^l}{\longrightarrow} I \stackrel{\ov{v}}{\longrightarrow} V$ and $\epsilon^r(\ov{v})$ by $V \otimes V^r \stackrel{\epsilon^r}{\longrightarrow} I \stackrel{\ov{v}}{\longrightarrow} V$. The diagrams for these  applications are as follows:
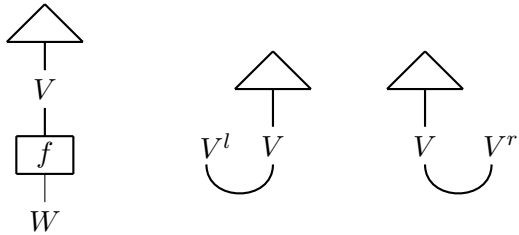
\begin{center}
  {%
\beginpgfgraphicnamed{comp-epsilon-app}
\InputIfFileExists{comp-epsilon-app.tikz}{}{\input{./tikz/comp-epsilon-app.tikz}}
\endpgfgraphicnamed}
\end{center} 

\smallskip
\noindent
Applications of the Frobenius maps to vectors are depicted in a similar fashion; for instance $\mu(\ov{v}\otimes \ov{v})$ is the composition $I \otimes I \stackrel{\ov{v} \otimes \ov{v}} {\longrightarrow} V \otimes V \stackrel{\mu}{\longrightarrow} V$ and $\iota(\ov{v})$ is the composition $I    \stackrel{\ov{v}}{\longrightarrow}V \stackrel{\iota}{\longrightarrow} I$, depicted as follows:
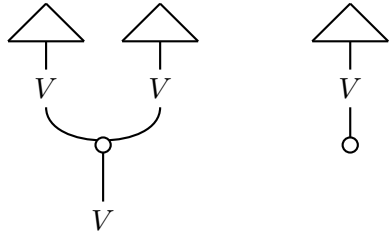
\begin{center}
  {%
\beginpgfgraphicnamed{comp-frob-app}
\InputIfFileExists{comp-frob-app.tikz}{}{\input{./tikz/comp-frob-app.tikz}}
\endpgfgraphicnamed}
\end{center}

\section{Vector Space Interpretations}

The grammatical structure of a language is encoded in the category
{\em Preg\/}: objects are grammatical types (assigned to words of the
language) and morphisms are grammatical reductions (encoding the
grammatical formation rules of the language). For instance, the
grammatical structure of the sentence ``Men love Mary''  is encoded in
the assignment of types $n$ to the noun phrases ``men'' and ``Mary''
and $n^r\otimes s\otimes n^l$ to the verb ``love'', and in the
reduction map $\epsilon_n^l \otimes 1_s \otimes \epsilon_n^r$.  The
application of this reduction map to the tensor product of the word types
in the sentence results in  the type $s$:
\[
(\epsilon_n^l \otimes 1_s \otimes \epsilon_n^r)(n \otimes (n^r \otimes s \otimes n^l) \otimes n) \to s
\]
\noindent
To each reduction map corresponds a
string diagram that depicts the structure of reduction:

 \begin{center}
  {%
\beginpgfgraphicnamed{preg-parse-eg}
\InputIfFileExists{preg-parse-eg.tikz}{}{\input{./tikz/preg-parse-eg.tikz}}
\endpgfgraphicnamed}
\end{center}

In  Coecke et al. (2010) \cite{Coeckeetal} the pregroup types and reductions are
interpreted as vector spaces and linear maps, achieved via a
homomorphic mapping from {\em Preg} to {\em FVect}. Categorically
speaking, this map is a strongly monoidal functor:

\[
F \colon \mbox{\em Preg} \to \mbox{\em FVect}
\]
\noindent
It assigns vector spaces to the basic types as follows:
\[
F(1) = I \qquad F(n)= N \qquad F(s) = S
\]
and to the compound types by monoidality as
follows; for $x,y$ objects of {\em Preg\/}:

\[
F(x\otimes y) = F(x) \otimes F(y)
\]
\noindent
Monoidal functors preserve the compact structure; that is the
following holds:

\[
F(x^l) = F(x^r) = F(x)^*
\]
\noindent
For instance, the interpretation of a transitive verb is computed as follows:
\begin{align*}
F(n^r \otimes s \otimes n^l) = F(n^r) \otimes F(s) \otimes F(n^l) =&  \\
 F(n)^* \otimes F(s) \otimes F(n)^* =  N\otimes S \otimes N&
 \end{align*}
 This interpretation means that the meaning vector of a
transitive verb is a vector in $N \otimes S \otimes N$.

The pregroup reductions, i.e. the partial order morphisms of {\em Preg},
are interpreted as linear maps: whenever $p \leq q$ in {\em Preg}, we have
a linear map $f_{\leq} \colon F(p) \to F(q)$. The $\epsilon$ and
$\eta$ maps of {\em Preg} are interpreted as the $\epsilon$ and $\eta$
maps of {\em FVect}. For instance,
the pregroup reduction of a transitive  verb sentence is computed as follows: 
\begin{align*}
F(\epsilon_n^r\otimes 1_s \otimes \epsilon_n^r) = F(\epsilon_n^r) \otimes F(1_s) \otimes
F(\epsilon_n^l) =&\\
F(\epsilon_n)^* \otimes F(1_s) \otimes
F(\epsilon_n)^* =  \epsilon_N \otimes 1_S \otimes \epsilon_N&
\end{align*}

The distributional meaning of a sentence is
obtained by applying the interpretation of the pregroup reduction of
the sentence to the tensor product of the distributional meanings of the
words in the sentence. For instance, the distributional
meaning of ``Men love Mary'' is as follows:
\[F(\epsilon_n^r \otimes 1_s
\otimes \epsilon_n^l)(\ov{\mbox{\em Men}} \otimes \ov{\mbox{\em love}} \otimes
\ov{\mbox{\em Mary}})\]
 This meaning is depictable via the following string diagram:

 \begin{center}
  {%
\beginpgfgraphicnamed{vect-sem-eg}
\InputIfFileExists{vect-sem-eg.tikz}{}{\input{./tikz/vect-sem-eg.tikz}}
\endpgfgraphicnamed}
\end{center}

\noindent Symbolically, the above diagram corresponds to $(\epsilon_N
\otimes 1_S \otimes \epsilon_N) (\ov{Men} \otimes \ov{love} \otimes
\ov{Mary})$, which is the application of the interpretation of the
pregroup reduction of the sentence to the tensor products of the
vectors of the words within it. Depending on their types, the
distributional meanings of the words are either atomic vectors or
linear maps. For instance, the distributional meaning of `men' is a
vector in $N$, whereas the distributional meaning of `love' is in the
space $N \otimes S \otimes N$, hence a linear map, in this case from
$N \otimes N$ to $S$.

The next section applies these techniques to the distributional
interpretation of pronouns. The interpretations are defined using:
$\epsilon$ maps, for application of the semantics of one word to
another; $\eta$ maps, to pass information around by bridging
intermediate words; and Frobenius operations, for copying and
 combining  the noun vectors and discarding the sentence vectors.

\section{Modelling Relative Pronouns}  

In this paper we focus on the subject and object relative pronouns,
{\em who(m)}, {\em which} and {\em that}. Examples of noun phrases
with subject relative pronouns are ``men who love Mary'', ``dog which
ate cats''.  Examples of noun phrases with object relative pronouns
are ``men whom Mary loves'', ``book that John read''. In the final
example, ``book'' is the head noun, modified by the relative clause
``that John read''.  The intuition behind the use of Frobenius
algebras to model such cases is the following. In ``book that John
read'', the relative clause acts on the noun (modifies it) via the
relative pronoun, which passes information from the clause to the
noun. The relative clause is then discarded, and the modified noun is
returned. Frobenius algebras provide the machinery for all of these
operations.

\noindent
The pregroup types of the relative pronouns are as follows:
\[
n^r n s^l n \quad \text{(subject)}\hspace{2cm}
n^r n n^{ll} s^l  \text{(object)}
\]
As argued in Lambek (2008) \cite{Lambek}, the occurrence of  double left adjoints  on the type of the object relative pronoun is a sign of the movement of the object from the end of the clause to its beginning. These types result in the
following reductions:

\hspace{-1.5cm}\begin{minipage}{6.8cm}
 {%
\beginpgfgraphicnamed{subj-syn-preg}
\InputIfFileExists{subj-syn-preg.tikz}{}{\input{./tikz/subj-syn-preg.tikz}}
\endpgfgraphicnamed}\end{minipage}
\quad   
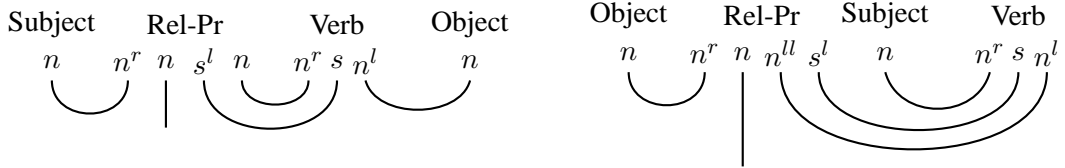
\begin{minipage}{6.8cm}
\vspace{0.3cm}
\hspace*{0.2cm} {%
\beginpgfgraphicnamed{obj-syn-preg}
\InputIfFileExists{obj-syn-preg.tikz}{}{\input{./tikz/obj-syn-preg.tikz}}
\endpgfgraphicnamed}
\end{minipage}

\noindent
The meaning spaces of these pronouns are computed using the 
mechanism described above:
\begin{eqnarray*}
F(n^rns^ln)  &=& F(n^r) \otimes F(n) \otimes F(s^l) \otimes F(n) \\ &=& N \otimes N \otimes S \otimes N \\
F(n^r n n^{ll} s^l) & = & F(n^r) \otimes F(n) \otimes F(n^{ll}) \otimes F(s^l) \\ &=& N \otimes N \otimes N \otimes S
\end{eqnarray*}

The semantic roles that these pronouns play are reflected in their
categorical vector space meanings, depicted as follows:
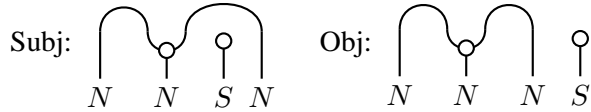
\begin{center}
\hspace*{-0.0cm}\mbox{Subj:}  {%
\beginpgfgraphicnamed{subj-rel-sem}
\InputIfFileExists{subj-rel-sem.tikz}{}{\input{./tikz/subj-rel-sem.tikz}}
\endpgfgraphicnamed}
\hspace{0.3cm}
\mbox{Obj:}  {%
\beginpgfgraphicnamed{obj-rel-sem}
\InputIfFileExists{obj-rel-sem.tikz}{}{\input{./tikz/obj-rel-sem.tikz}}
\endpgfgraphicnamed}
\end{center}
\vspace*{0.3cm}
\noindent
with the following corresponding morphisms:
\[
\mbox{Subj:} \;\; (1_N \otimes \mu_N \otimes \zeta_S \otimes 1_N) \circ (\eta_N  \otimes \eta_N)
\]
\[
\mbox{Obj:} \;\; (1_N \otimes \mu_N \otimes 1_N \otimes  \zeta_S) \circ (\eta_N  \otimes \eta_N)
\]
The details of these compositions are:
\[
\mbox{Subj:} \quad I \cong I \otimes I  \stackrel{{\eta_N \otimes \eta_N}} {\longrightarrow} N \otimes N \otimes N \otimes N
\cong N \otimes N \otimes N \otimes I \otimes N 
 \stackrel{1_N \otimes \mu_N \otimes \zeta_S \otimes 1_N} {\longrightarrow}
 N \otimes N \otimes S \otimes N
\] 
\[
\mbox{Obj:} \quad I \cong I \otimes I  \stackrel{{\eta_N \otimes \eta_N}} {\longrightarrow} N \otimes N \otimes N \otimes N
\cong N \otimes N \otimes N  \otimes N \otimes I
 \stackrel{1_N \otimes \mu_N \otimes1_N \otimes \zeta_S} {\longrightarrow}
 N \otimes N  \otimes N \otimes S
\]

The diagram of the meaning vector of the subject relative clause
interacting with the head noun is as follows:

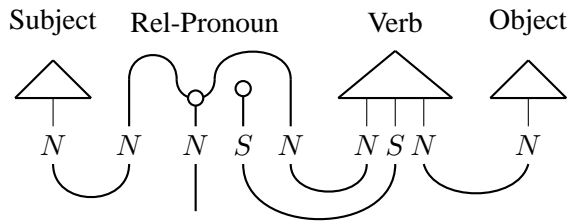
\begin{center}\begin{minipage}{7cm}
  {%
\beginpgfgraphicnamed{subj-sem}
\InputIfFileExists{subj-sem.tikz}{}{\input{./tikz/subj-sem.tikz}}
\endpgfgraphicnamed}\end{minipage}
\end{center}  
 
\noindent
The diagram for the object relative clause is:

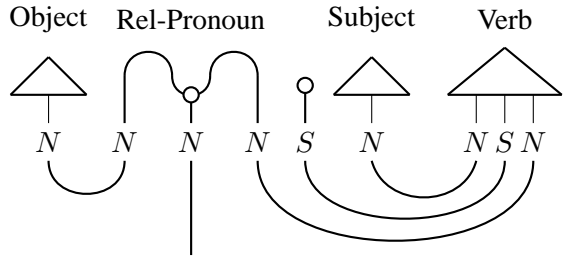
\begin{center}
\begin{minipage}{7cm}
   {%
\beginpgfgraphicnamed{obj-sem}
\InputIfFileExists{obj-sem.tikz}{}{\input{./tikz/obj-sem.tikz}}
\endpgfgraphicnamed}
\end{minipage}
\end{center}
These diagrams depict the flow of information in a relative clause and
the semantic role of its relative pronoun, which 1) passes information
from the clause to the head noun via the $\eta$ maps; 2) acts on the
noun via the $\mu$ map; 3) discards the clause via the $\zeta$ map;
and 4) returns the modified noun via $1_N$.  The $\epsilon$ maps pass
the information of the subject and object nouns to the verb and to the
relative pronoun to be acted on.  Note that there are two different
flows of information in these clauses: the ones that come from the
grammatical structure and are depicted by $\epsilon$ maps (at the
bottom of the diagrams), and the ones that come from the semantic role
of the pronoun and are depicted by $\eta$ maps (at the top of the
diagrams).

The normal forms of these diagrams are:

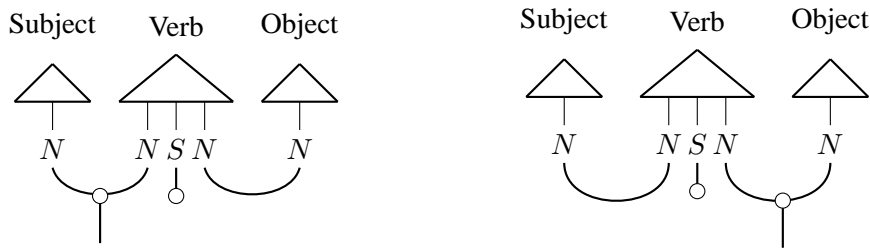
\begin{center}
   {%
\beginpgfgraphicnamed{subj-sem-norm}
\InputIfFileExists{subj-sem-norm.tikz}{}{\input{./tikz/subj-sem-norm.tikz}}
\endpgfgraphicnamed}
\hspace{2cm}
   {%
\beginpgfgraphicnamed{obj-sem-norm}
\InputIfFileExists{obj-sem-norm.tikz}{}{\input{./tikz/obj-sem-norm.tikz}}
\endpgfgraphicnamed}
\end{center}
\noindent
Symbolically, they correspond to the following morphisms:
\[
\left(\mu_N \otimes \iota_S \otimes \epsilon_N\right) \left(\ov{\text{Subject}} \otimes \ov{\text{Verb}} \otimes \ov{\text{Object}}\right)\]
\vspace*{-0.6cm}
\[ \left(\epsilon_N \otimes \iota_S \otimes \mu_N\right) \left(\ov{\text{Subject}} \otimes \ov{\text{Verb}} \otimes \ov{\text{Object}}\right)
\]
\noindent
The simplified normal forms will become useful in practice when
calculating vectors for such cases.

\section{Truth Theoretic Instantiations}
\label{truth-vect}

In this section we demonstrate the effect of the Frobenius operations
in a truth-theoretic setting, similar to Coecke et al. (2010) \cite{Coeckeetal} but also allowing for degrees of truth.   This instantiation  is designed
only as a theoretical example, which merely recasts Montague semantics,  rather than a suggestion for implementation.   A suggestion for  concrete implementation is presented in Section \ref{conc}.

Take $N$ to be the vector space spanned by a set of individuals
$\{\ov{n}_i\}_i$  that are mutually orthogonal. For example, $\ov{n}_1$ represents the individual
Mary, $\ov{n}_{25}$ represents Roger the dog, $\ov{n}_{10}$ represents
John, and so on.  A sum of basis vectors in this space represents a
common noun; e.g. $\ov{man} = \sum_i \ov{n}_i$, where $i$ ranges over
the basis vectors denoting men. We take
$S$ to be the one dimensional space spanned by the single vector $\ov{1}$.  The
unit vector spanning $S$ represents truth value 1, the zero vector
represents truth value 0, and the intermediate vectors represent
degrees of truth.

A transitive verb $w$, which is a vector in the space $N \otimes S
\otimes N$, is represented as follows:
\[
  \overline{w} := \sum_{ij} \ov{n}_i \otimes (\alpha_{ij} \ov{1})
  \otimes \ov{n}_j
  \]
\noindent
if $\ov{n}_i \ w\mbox{'s} \ \ov{n}_j \ \mbox{with degree}
\ \alpha_{ij}$, for all $i,j$.

Further, since $S$ is one-dimensional with its only basis vector being
$\ov{1}$, the transitive verb can be represented by the following element of $N
\otimes N$:
\[
\sum_{ij}  \alpha_{ij}  \ov{n}_i \otimes\ov{n}_j \quad \text{if} \quad
\ov{n}_i \ w\mbox{'s} \ \ov{n}_j \ \mbox{with degree}
\ \alpha_{ij}
\]
\noindent
Restricting to either   $\alpha_{ij} = 1$ or $\alpha_{ij} = 0$ provides a 0/1 meaning, i.e.  either
$\ov{n}_i$ \ w's \ $\ov{n}_j$ or not. Letting $\alpha_{ij}$ range over  the interval 
$[0,1]$ enables us to represent degrees as well as limiting cases of
truth and falsity.  For example, the verb ``love'', denoted by $\overline{love}$, is represented by:
\[
 \sum_{ij}\alpha_{ij}  \ov{n}_i \otimes 
\ov{n}_j \quad \text{if} \quad \ov{n}_i \ \text{loves} \ \ov{n}_j \mbox{with degree} \ 
\alpha_{ij}
\]
If we take $\alpha_{ij}$ to be $1$ or $0$, from the above we obtain the following:
\[
\sum_{(i,j) \in  R_{love}} \ov{n}_i
\otimes \ov{n}_j
\]
  where  $R_{love}$  is the set of all pairs $(i,j)$ such that $\ov{n}_i$ loves $\ov{n}_j$.  Note that, with this definition, the sentence space has already been
discarded, and so for this instantiation the $\iota$ map, which is the
part of the relative pronoun interpretation designed to discard the
relative clause after it has acted on the head noun, is not required.

For common nouns $\ov{\text{Subject}} = \sum_{k\in K} \ov{n}_k$ and
$\ov{\text{Object}} = \sum_{l\in L} \ov{n}_l$, where $k$ and $l$ range
over the sets of basis vectors representing the respective common
nouns, the truth-theoretic meaning of a noun phrase modified by a
subject relative clause is computed as follows:

\begin{align*}
&\ov{\mbox{Subject who Verb Object}} := (\mu_N \otimes  \epsilon_N) \left(\ov{\text{Subject}} \otimes
 \ov{\text{Verb}} \otimes  \ov{\text{Object}} \right)\\
& = (\mu_N \otimes  \epsilon_N)\!\! \left(\sum_{k\in K} \ov{n}_k \otimes\! (\sum_{ij} \alpha_{ij} \ov{n}_i \!\otimes\! \ov{n}_j)\! \otimes\!  \sum_{l\in L} \ov{n}_l \!\!\right) \\
 &= \sum_{ij,k\in K, l \in L}  \alpha_{ij} \mu_N(\ov{n}_k \otimes\ov{n}_i)
\otimes   \epsilon_N(\ov{n}_j \otimes \ov{n}_l)\\
 &=\sum_{ij,k\in K, l \in L} \alpha_{ij}  \delta_{ki}\ov{n}_i 
 \delta_{jl}\\
&=\sum_{k\in K, l \in L}  \alpha_{kl} \ov{n}_k 
\end{align*}

 The result is highly intuitive, namely the sum of the subject individuals
weighted by the degree with which they have acted on the object
individuals via the verb.  A similar computation, with the difference
that the $\mu$ and $\epsilon$ maps are swapped, provides the
truth-theoretic semantics of the object relative clause:
\[  \sum_{k\in K, l \in L}  \alpha_{kl} \ov{n}_l
\]
The calculation and final outcome is best understood with an example.

Now only consider truth values 0 and 1.  Consider
the noun phrase with object relative clause ``men whom Mary loves''
and take $N$ to be the vector space spanned by the set of all people;
then the males form a subspace of this space, where the basis vectors
of this subspace, i.e. men, are denoted by $\ov{m}_{l}$, where $l$
ranges over the set of men which we denote by $M$. We set ``Mary'' to
be the individual $\ov{f}_1$, ``men'' to be the common noun
$\sum_{l\in M}\ov{m}_l$, and ``love'' to be as follows:
\[
\sum_{(i,j) \in  R_{love}}
\ov{f}_i \otimes \ov{m}_j
\]
  The vector corresponding to the meaning
of ``men whom Mary loves'' is computed as follows and as expected,  the result is the sum of the men basis
vectors which are also loved by Mary.

\begin{align*} 
&\ov{\mbox{men whom Mary loves}} :=  \left(\epsilon_N  \otimes \mu_N\right) \left(\ov{f}_1 \otimes (\!\!\!\!\sum_{(i,j) \in  R_{love}}\!\!\!\! \ov{f}_i \otimes \ov{m}_j) \otimes \sum_{l\in M} \ov{m}_l \right)\\
& =  \sum_{l \in M, (i,j) \in  R_{love}}\!\!\!\!\!\! \epsilon_N(\ov{f}_1 \otimes \ov{f}_i)\otimes
 \mu(\ov{m}_j \otimes \ov{m}_l)\\
& =\sum_{l \in M, (i,j) \in  R_{love}} \!\!\!\!\!\!  \delta_{1 i} \delta_{j l}\ov{m}_j\\
& = \sum_{(1,j) \in  R_{love} | j\in M}  \!\!\!\!\!\! \ov{m}_j
\end{align*}

A second example involves degrees of truth.  Suppose we have two
females Mary $\ov{f}_1$ and Jane $\ov{f}_2$ and four men $\ov{m}_1,
\ov{m}_2, \ov{m}_3, \ov{m}_4$. Mary loves $\ov{m}_1$ with degree 1/4
and $\ov{m}_2$ with degree 1/2; Jane loves $\ov{m}_3$ with
degree 1/5; and $\ov{m}_4$ is not loved. In this situation,  we have:
\[
R_{love}= \{(1,1), (1,2), (2,3)\}
\]
and the verb love is represented by:
\[
1/4 (\ov{f}_1 \otimes \ov{m}_1) + 1/2 (\ov{f}_1 \otimes \ov{m}_2) +  1/5 (\ov{f}_2 \otimes \ov{m}_3)
\]
The meaning of ``men whom Mary loves'' is computed by substituting an
$\alpha_{1,j}$ in the last line of the above computation, 
resulting in the men whom Mary loves together with the degrees that
she loves them:
\[
 \sum_{(1,j) \in  R_{love}  | j\in M} \!\!\!\!\!\!\!\!\!\!\alpha_{1j} \ov{m}_j = 1/4 \ov{m}_1 + 1/2 \ov{m}_2 
\]
The meaning of the clause ``men whom women love'' is computed as follows, where $W$ is the set of women:
\begin{align*} 
 &  \sum_{k\in W,l\in M, (i,j) \in  R_{love} } \!\!\!\!\!\! \!\!\!\!\!\!\alpha_{ij} \epsilon_N(\ov{f}_k \otimes \ov{f}_i)\otimes
 \mu(\ov{m}_j \otimes \ov{m}_l)\\
 &=\sum_{ k\in W,l\in M, (i,j) \in  R_{love} } \!\!\!\!\!\! \!\!\!\!\!\!\alpha_{ij} \delta_{k i} \delta_{j l}\ov{m}_j\\
& = \sum_{(i,j) \in  R_{love} |i\in W,j\in M} \!\!\!\!\!\! \!\!\!\!\!\!\alpha_{ij} \ov{m}_j \\
 &= 1/4 \ov{m}_1 + 1/2 \ov{m}_2 + 1/5 \ov{m}_3
\end{align*}
The result is the men loved by Mary or Jane together
with the degrees to which they are loved.

As a more complicated example, consider the clause   "Movies that Mary liked which became famous", the first pronoun `that' has an objective role and the second pronoun `which' has a subjective role. The diagram of this clause is as follows:
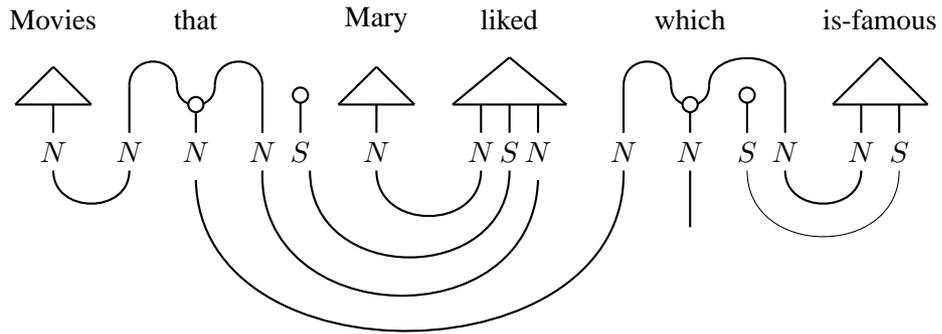
\begin{center}
    {%
\beginpgfgraphicnamed{fuse-1}
\InputIfFileExists{fuse-1.tikz}{}{\input{./tikz/fuse-1.tikz}}
\endpgfgraphicnamed}
\end{center}
This first normalises to the following:
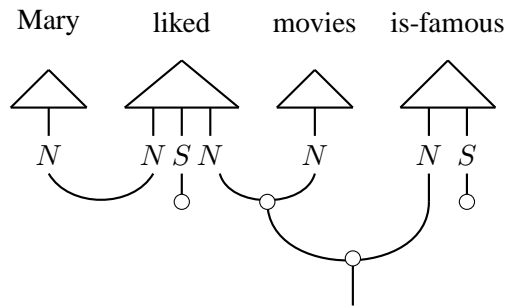
\begin{center}
    {%
\beginpgfgraphicnamed{fuse-1-simple-1}
\InputIfFileExists{fuse-1-simple-1.tikz}{}{\input{./tikz/fuse-1-simple-1.tikz}}
\endpgfgraphicnamed}
\end{center}
and then by the spider normal form,  the two $\mu$ maps fuse and we obtain the following normal form:
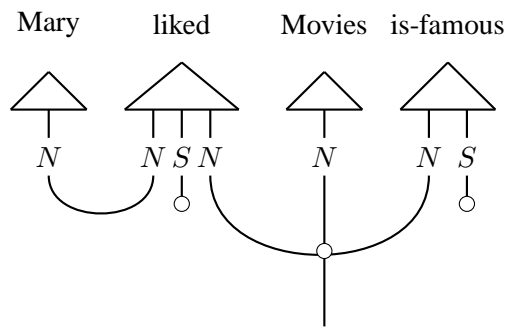
\begin{center}
    {%
\beginpgfgraphicnamed{fuse-1-simple-2}
\InputIfFileExists{fuse-1-simple-2.tikz}{}{\input{./tikz/fuse-1-simple-2.tikz}}
\endpgfgraphicnamed}
\end{center}
The vector space meaning of this clause is computed according to the first normalisation above (that is, before the application of the spider normal form) to demonstrate how the composition of two $\mu$ maps works (we also drop the $\iota$ map as our concrete truth-theoretic construction does not need it). The result is  a weighted sum of the famous movies that Mary loves, where each  weight is  the multiplications of the degree to which Mary loves a movie   with the degree to which the movie is famous. 

We may also have cases where the two relative pronouns do not compose, for instance in  "men who love books that Mary wrote". The  diagram for this sentence is as follows:
\begin{center}
    {%
\beginpgfgraphicnamed{no-fuse-1}
\InputIfFileExists{no-fuse-1.tikz}{}{\input{./tikz/no-fuse-1.tikz}}
\endpgfgraphicnamed}
\end{center}
Here the    $\mu$ maps do not compose: the first $\mu$ is applied to the subject of "loves", whereas the second one is applied to the object of "wrote"  and these two dimensions do not interact. The following simplified form of the above diagram makes this point clearer:
\begin{center}
    {%
\beginpgfgraphicnamed{no-fuse-1-simple}
\InputIfFileExists{no-fuse-1-simple.tikz}{}{\input{./tikz/no-fuse-1-simple.tikz}}
\endpgfgraphicnamed}
\end{center}
Here, the existing $\mu$ and $\iota$'s  do not interact either: the first pair filters out men who love books from the set of all men and the second pair filters out books that Mary wrote from the set of all books. In the Boolean case, the result will be the sum of men who love books written by Mary. In the case of degrees of love, the result will be a weighted sum of men, each with the degree to which they love these books. 

A a  final example, consider the sentence  `Mary loves men whom Mary loves', which is a tautology, i.e. it is always true.  The pictorial meaning of this sentence is as follows:
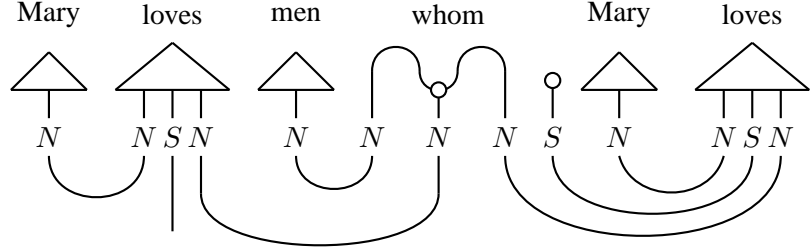
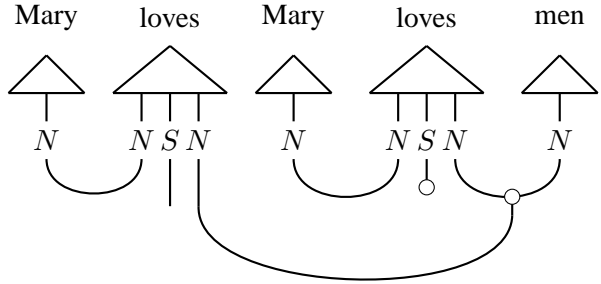
\begin{center}
   \qquad {%
\beginpgfgraphicnamed{long-sentence}
\InputIfFileExists{long-sentence.tikz}{}{\input{./tikz/long-sentence.tikz}}
\endpgfgraphicnamed}
\end{center}

\noindent which  normalises  to
\begin{center}
   \qquad {%
\beginpgfgraphicnamed{short-sentence}
\InputIfFileExists{short-sentence.tikz}{}{\input{./tikz/short-sentence.tikz}}
\endpgfgraphicnamed}
\end{center}

\noindent The vector space meaning of this sentence becomes as follows:
\begin{eqnarray*}
\sum_{ij}\langle \ov{f}_1 \mid  \ov{f}_i \rangle \,  \langle \ov{m}_j \mid   \sum_{ij}  \delta_{1i}  \  \delta_{lj} \  \ov{m}_j\rangle \end{eqnarray*}
where $ \sum_{ij}  \delta_{1i}  \  \delta_{lj} \  \ov{m}_j$ is `men whom Mary loves', as computed above. So the inner product $ \langle \ov{m}_j \mid   \sum_{ij}  \delta_{1i}  \  \delta_{lj} \  \ov{m}_j\rangle $ only  returns   the men for which   'love Mary' is true, resulting aways  in the output  1.

\section{Embedding Predicate Semantics}
In this section, we provide a set-theoretic interpretation for relative  clauses according to the constructions discussed in \cite{Sag}. We start by fixing   a  universe of elements ${\cal U}$. A proper nouns is an individual (i.e. an element) in  this set;  a common noun is  the set of the individuals that have the property expressed by the common noun;  hence common nouns are unary predicates over ${\cal U}$.  An intransitive verb is the set of all individuals that are acted upon by the relationship expressed by the verb, hence these are  unary predicates over  ${\cal U}$. A transitive verb is the set of pairs of individuals that are related by the relationship expressed by the verb; hence   these  are binary predicates over ${\cal U} \times {\cal U}$.  
Here is an  example for each case:
\begin{align*}
\semantics{Mary} &= \{u_1\}\\
\semantics{Author} &= \{x \in {\cal U} \mid "x \  \mbox{is an author}"\}\\
\semantics{Sleep} &= \{x \in {\cal U}  \mid "x  \  \mbox{sleeps}"\}\\
\semantics{Entertain} &= \{(x,y) \in {\cal U} \times {\cal U} \mid "x \  \mbox{entertains} \ y"\}
\end{align*}
The subjective and objective relative clauses are interpreted as follows:
\begin{eqnarray*}
&&\left\{x \in {\cal U}   \mid x \in \semantics{Subj},  (x, y) \in \semantics{Verb},  \  \mbox{for all} \ y \in\semantics{Obj}  \right\}\\
&&
\left\{y  \in {\cal U} \mid y \in \semantics{Obj},   (x, y) \in \semantics{Verb},   \ \mbox{for all} \ x \in \semantics{Subj} \right\}
\end{eqnarray*}
These clauses pick out individuals which belong to the interpretation of  subject/object and which  are in the relationship expressed by the verb with all elements of the interpretation of object/subject.  Examples are  `men who love Mary' and `men whom cats loves', interpreted as follows, for $x \in {\cal U}$:
\begin{eqnarray*}
&&\left\{x \in {\cal U} \mid x \in \semantics{men},   (x, y) \in \semantics{Love},  \ \mbox{for all} \ y \in \semantics{Mary}\right\} \\
&&\left\{y \in {\cal U}  \mid y \in \semantics{men},   (x, y) \in \semantics{Love},  \ \mbox{for all} \ x \in \semantics{cat}\right\} 
\end{eqnarray*}
The first example picks out individuals that are `men' and who are in relationship of `love' with the individual `Mary'. The second one picks out individuals that are `men' and who are `loved' by  all the individual that are  ` cats'. 

We obtain  the vector space forms of the above predicate interpretations  by developing a map from the category of sets and relations to the category of finite dimensional vector spaces  with a fixed orthogonal basis and linear maps over them, that is a map with types as follows:
\[
e\colon Rel \hookrightarrow FVect
\]
We set this  map to  turn a set $T$  into a vector space $V_T$ spanned by the elements of $T$ and an element $t \in T$   into a basis vector $\ov{t}$ of $V_T$. A subset $W$ of $T$ is turned into  the sum vector of its elements, i.e. $\sum_i \ov{w}_i$, where $w_i$'s enumerate over  elements of $W$.  A relation $R \subseteq T \times T'$ is turned into  a linear map  from $T$ to $T' $, or equivalently  as an element of the space   $T \otimes T'$;  we represent this element by the sum of its basis vectors $\sum_{ij} \ov{t}_i \otimes \ov{t'}_j$, here $t_i$'s enumerate  elements of  $T$ and $t'_j$'s enumerate  elements of $T'$. 

We then use   the compact structure of $Rel$ and $FVect$ and turn the application of a relation  $R \subseteq T \times T'$  to its arguments   into   the inner product of the vector space forms of the relation and the arguments. In our sum representation, $R(a)$ when $a \in T$ and $R^{-1}(b)$ when $b \in T'$ are turned into  the following: 
\[
\sum_{ij} \langle a \mid \ov{t}_i \rangle \ov{t'}_j 
\hspace{2cm}
\sum_{ij}  \ov{t}_i \langle  \ov{t'}_j \mid b \rangle
\]
To check wether an element $t$ is in the set $T$, we  use the Frobenius  operation $\mu$ on $T$, that is $\mu \colon T \otimes T \to T$,  as follows:
\[
t \in T \quad \mbox{whenever} \quad \mu(\ov{t}, \sum_i \ov{t}_i) = \ov{t}
\]
The intersection of two sets $T, T'$ is computed by generalising the above  via  the Frobenius $\mu$ operation on  the whole universe, that is $\mu \colon {\cal U} \otimes {\cal U} \to {\cal U}$, or on a set  containing both of these sets, e.g. $T \cup T'$, that is  $\mu \colon (T \cup T') \otimes (T \cup T') \to (T \cup T')$.  In either case, for $\sum_i \ov{t}_i$ and $\sum_j \ov{t'}_j$ representations of $T$ and $T'$ we obtain the following for the intersection:
\[
T \cap T'  \ := \ \mu(\sum_i \ov{t}_i, \sum_j \ov{t'}_j)
\]
We are now ready to show that the vector space forms of the predicate interpretation of relative clauses, developed along side the constructions described above,  provide us with the same semantics as the truth-theoretic vector space semantics developed in Section \ref{truth-vect}. That is, we have:
\begin{proposition}
The map $e$, described above,  provides a 0/1 truth-theoretic interpretation for relative clauses. 
\end{proposition}
\begin{proof}
The first step is to instantiate the proper and common nouns, the intransitive and transitive verbs. For this, we set the universe $\cal U$ to be the  individuals representing  the nouns of the language, then the  map $e$  instantiates as follows:
\begin{itemize}
\item The universe ${\cal U}$ becomes the vector space $V_{\cal U}$, spanned by its elements. So an element $u_i \in {\cal U}$ becomes a basis vector $\ov{u}_i$ of $V_{\cal U}$.  
\item A proper noun $a \in {\cal U}$ becomes a basis vector $\ov{u}_a$ of $V_{\cal U}$. 
\item A common noun $c \subset {\cal U}$ becomes a  subspace of $V_{\cal U}$, represented by the sum of  the representations of its elements $\ov{u}_i$, that is
\[
c \hookrightarrow \ov{c} \quad :: \quad \sum_i \ov{u}_i \quad \text{for \ all} \ u_i \in c
\]
\item An intransitive verb $v \subset {\cal U}$ becomes a subspace of $V_{\cal U}$, represented by the sum of all its  elements $u_i$, that is
\[
v \hookrightarrow \ov{v} \quad ::\quad \sum_i \ov{u}_i \quad \text{for \ all} \  v_i \in v
\]
\item A transitive verb $w \subset {\cal U} \times {\cal U}$ becomes a linear map $\overline{w}$ in $V_{\cal U} \otimes V_{\cal U}$, represented by the sum of its basis vectors  $\ov{u}_i, \ov{u}_j$, which in this case are pairs of its elements $(u_i, u_j)$, that is
\[
w \hookrightarrow \overline{w} \quad :: \quad \sum_{ij} \ov{u}_i \otimes \ov{u}_j \quad \text{for \ all} \ (u_i, u_j) \in w
\]
\end{itemize}
The next step is to use the above definitions and develop the vector space forms of  the predicate interpretations of the relative clauses. To do so, we turn these interpretations into intersections of subsets and check the membership relation  via the $\mu$ map. As for  the  application of the predicates to the interpretation of their arguments, we use their relational image on subsets. That is,  for $R \subseteq {\cal U} \times {\cal U}$ and $T \subseteq {\cal U}$, we work with $R[T]$, defined by $R[T] = \{t' \in {\cal U}\mid \  (t, t') \in R \ \text{for}  \ t \in T\}$. This form of application has an implicit universal quantification: it consists of all the elements of ${\cal U}$ that are related to some element of $T$. Hence, the intersection forms  take care of  the quantification present in the predicate semantics of the relative clauses without having to explicitly use it.  

Consider the subjective clause, its predicate interpretation is equivalent to the following intersection of subsets: 
\[
 \semantics{Subj} \ \cap \ \semantics{Verb}^{-1}[\semantics{Obj}] 
\]
The vector space form of this intersection is the application of the $\mu$ map to the  vector space forms of each of the  subsets involved in it. For the latter, we use their sum representation. That is, for $\sum_i \ov{u}_i$ and $\sum_j \ov{u}_j$ the representation of subspaces $V_{\semantics{Subj}}$ and $V_{\semantics{Verb}^{-1}[\semantics{Obj}] }$,  the vector space form of the above set  is  obtained as follows: 
\[
V_{ \semantics{Subj} \ \cap \ \semantics{Verb}^{-1}[\semantics{Obj}] } \ = \ 
 \mu\left(V_{\semantics{Subj}}, V_{\semantics{Verb}^{-1}[\semantics{Obj}] }\right) \ = \ 
\mu\left(\sum_i \ov{u}_i, \sum_j \ov{u}_j\right) 
\]
Recall that relational application  was modelled by taking inner products, so we obtain $\sum_j \ov{u}_j$  by taking the inner product of  vector space form of the verb to the vector form of the object, that is we have:
\[
V_{\semantics{Verb}^{-1}[\semantics{Obj}]}  \ = \ \sum_j \ov{u}_j = \sum_{ij} \ov{u}_i \langle \ov{u}_j \mid \sum_{b \in V_{\semantics{Obj}}}\ov{u}_b\rangle
\]
Substituting this term for $\sum_j \ov{u}_j$ in  the $\mu$ term above  provides us with the  truth-theoretic meaning of the same clause in vector spaces. For the other clauses, we follow a similar procedure on the intersection form of their  predicate semantics. The case for  the objective relative clause, whose intersection form is $\semantics{Obj} \cap \semantics{Verb}[\semantics{Subj}]$ is almost immediate.  
\end{proof}

\section{Concrete Instantiations}
\label{conc}

In the model of  Grefenstette and
Sadrzadeh (2011a) \cite{GrefenSadr1}, the meaning of a verb is taken
to be ``the degree to which the verb relates properties of its
subjects to properties of its objects''. Clark (2013) \cite{clark:chapter13}
provides some examples showing how this is an intuitive definition for
a transitive verb in the categorical framework.  This degree is
computed by forming the sum of the tensor products of the subjects and
objects of the verb across a corpus, where $w$ ranges over instances
of the verb:
\[
\overline{\text{verb}} = \sum_w (\ov{\text{sbj}} \otimes \ov{\text{obj}})_w
\]
We denote the vector space of nouns by $N$; the above is a matrix in $N
\otimes N$, depicted by a two-legged triangle as follows:

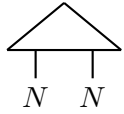
\begin{center}
  {%
\beginpgfgraphicnamed{plain-verb}
\InputIfFileExists{plain-verb.tikz}{}{\input{./tikz/plain-verb.tikz}}
\endpgfgraphicnamed}
\end{center}

The verbs of this model do not have a sentence dimension; hence no
information needs to be discarded when they are used in our setting,
and so no $\iota$ map appears in the diagram of the relative
clause. Inserting the above diagram in the diagrams of the normal
forms results in the
following for the subject relative clause (the object case is similar):

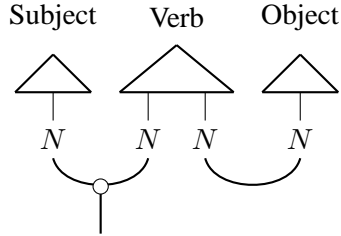
\begin{center}
   {%
\beginpgfgraphicnamed{subj-sem-norm-con}
\InputIfFileExists{subj-sem-norm-con.tikz}{}{\input{./tikz/subj-sem-norm-con.tikz}}
\endpgfgraphicnamed}
\end{center}

The abstract vectors corresponding to such diagrams are similar to the
truth-theoretic case, with the difference that the vectors are
populated from corpora and the scalar weights for noun vectors are not
necessarily 1 or 0. For subject and object noun context
vectors computed from a corpus as follows:
\[
\ov{\text{Subject}} = \sum_k C_k \ov{n}_k \qquad
\ov{\text{Object}} = \sum_l C_l \ov{n}_l
\]
and the verb a linear map:
\[
\overline{\text{Verb}} =
\sum_{ij} C_{ij} \ov{n}_i \otimes \ov{n}_j
\]
computed as above,  the
concrete meaning of a noun phrase modified by a subject relative
clause is as follows:
\begin{align*}
&\sum_{kijl} C_k C_{ij} C_l \mu_N(\ov{n}_k \otimes \ov{n}_i) \epsilon_N(\ov{n}_j \otimes \ov{n}_l)\\
&= \sum_{kijl} C_kC_{ij}C_l \delta_{ki} \ov{n}_k \delta_{jl}\\
&= \sum_{kl} C_k C_{kl} C_l \ov{n}_k
\end{align*}
Comparing this to the truth-theoretic case, we see that the previous
$\alpha_{kl}$ are now obtained from a corpus and instantiated to
$C_k C_{kl} C_l$. To see how the above expression represents the
meaning of the noun phrase, decompose it into the following:
\[
\sum_{k}C_k \ov{n}_k
\odot \sum_{kl} C_{kl} C_l \ov{n}_l
\]
Note that the second term of the above,
which is the application of the verb to the object, modifies the
subject via point-wise multiplication. A similar result arises for the
object relative clause case.

As an example, suppose that $N$ has two dimensions with basis vectors
$\ov{n}_1$ and $\ov{n}_2$, and consider the noun phrase ``dog that
bites men''.  Define the vectors of ``dog'' and ``men'' as follows:
\[
\ov{\text{dog}} = d_1 \ov{n}_1 + d_2 \ov{n}_2\qquad
\ov{\text{men}} = m_1 \ov{n}_1 + m_2 \ov{n}_2\qquad
\]
and the matrix of ``bites'' by:
\[ 
 b_{11} \ov{n}_1 \otimes \ov{n}_2 + b_{12} \ov{n}_1 \otimes \ov{n}_2 + b_{21} \ov{n}_2 \otimes \ov{n}_1 + b_{22} \ov{n}_2 \otimes \ov{n}_2
\]
 Then the meaning of the noun phrase becomes:
\begin{align*}
&\ov{\mbox{dog that bites men}} :=\\
 &   d_1b_{11} m_1 \ov{n}_1 +  d_1 b_{12} m_2  \ov{n}_1 + d_2 b_{21} m_1 \ov{n}_2   + d_2 b_{22} m_2 \ov{n}_2 \ = \\
&(d_1 \ov{n}_1 + d_2 \ov{n}_2) \odot  \left((b_{11} m_1  +  b_{12} m_2)  \ov{n}_1 +  (b_{21} m_1  + b_{22} m_2) \ov{n}_2\right)
\end{align*}
 Using matrix notation, we can decompose the second term
further, from which the application of the verb to the object becomes
apparent:
\[ 
\left(\begin{array}{cc}
b_{11} & b_{12}\\
b_{21} & b_{22}
\end{array}\right)  \ \times \ \left(\begin{array}{c}m_1\\m_2\end{array}\right)
\]
Hence for the whole clause we obtain:
\[
\ov{\text{dog}} \odot (\overline{\text{bites}} \times \ov{\text{men}})
\]

Again this result is highly intuitive: assuming that the basis vectors
of the noun space represent properties of nouns, the meaning of ``dog
that bites men'' is a vector representing the properties of dogs, which
have been modified (via multiplication) by those properties of
individuals which bite men. Put another way, those properties of dogs
which overlap with properties of biting things get accentuated.

In  the same lines, one can also depict the model of \cite{GrefenSadr2}, where   a matrix is built for the verb by taking the tensor product of the context vector of the verb with itself. Assuming that the context vector of  verb $w$   is $\ov{w}$,  this matrix is given and depicted as follows:
\[
\overline{w}  =  \ov{w} \otimes \ov{w}
\qquad  \mbox{depicted by} \qquad {%
\beginpgfgraphicnamed{Kron-Verb}
\InputIfFileExists{Kron-Verb.tikz}{}{\input{./tikz/Kron-Verb.tikz}}
\endpgfgraphicnamed}
\]

\noindent  This instantiation blocks the flow of information from the subject to the object  of the verb, and moreover   either the application of the verb to the subject or to the object will be a scalar. The case of the subjective relative clause, where the application of the verb to its object becomes a scalar,  is depicted bellows:

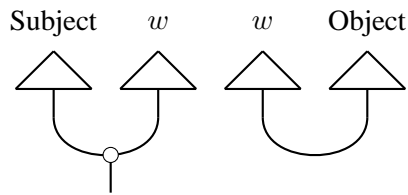
\begin{center}
  {%
\beginpgfgraphicnamed{Kron}
\InputIfFileExists{Kron.tikz}{}{\input{./tikz/Kron.tikz}}
\endpgfgraphicnamed}
\end{center}

\noindent Most (if not all) of the distributional language tasks are based on the distance between the vectors with the cosine of the angle  proven to be  a good measure of word  similarity. The cosine of the angle between two vectors is proportional to their inner product, hence it can be depicted using string diagrams, as follows:

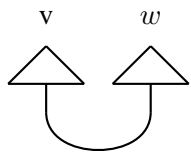
\begin{center}
  {%
\beginpgfgraphicnamed{Angle}
\InputIfFileExists{Angle.tikz}{}{\input{./tikz/Angle.tikz}}
\endpgfgraphicnamed}
\end{center}

\noindent The problem with this  model is that it will result in the same angle between, for instance, subjective clauses that have different objects, hence ignoring the object part all together (in the objective clauses the  subject part will be ignored), as shown below:

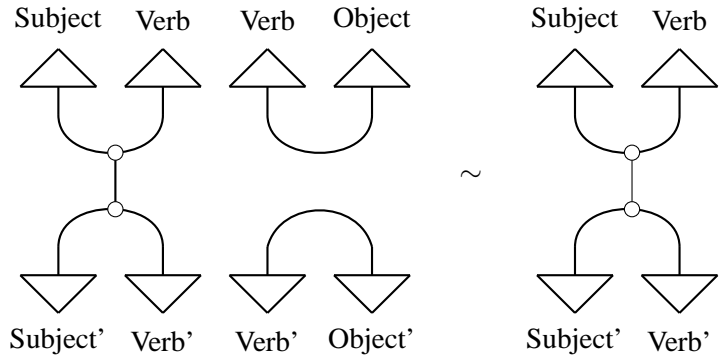
\begin{center}
  {%
\beginpgfgraphicnamed{Kron-Subj}
\InputIfFileExists{Kron-Subj.tikz}{}{\input{./tikz/Kron-Subj.tikz}}
\endpgfgraphicnamed} \quad $\sim$ \quad  {%
\beginpgfgraphicnamed{Kron-Subj2}
\InputIfFileExists{Kron-Subj2.tikz}{}{\input{./tikz/Kron-Subj2.tikz}}
\endpgfgraphicnamed}  
\end{center}

\noindent Here, the application of the verb to the object results in a scalar, hence  the  angle between "Subj Verb Obj" and "Subj' Verb' Obj' " will become equal to the angle between   "Subj Verb" and "Subj' Verb' ''.  A similar problem arises for other clauses.

Another  possibility is the models  of \cite{Kartetal2},  which are obtained by starting from   the  model of \cite{GrefenSadr1}, then either copying  the subject or the object dimension of the verb, depicted as follows: 
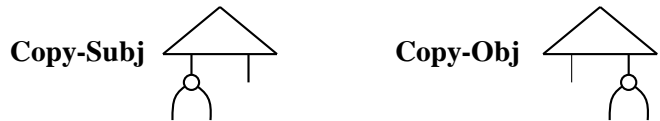
\begin{center}
  {%
\beginpgfgraphicnamed{copy-sbj-verb-rel}
\InputIfFileExists{copy-sbj-verb-rel.tikz}{}{\input{./tikz/copy-sbj-verb-rel.tikz}}
\endpgfgraphicnamed}
\end{center}
The problem here is that both of these models result in the same vector for relative clauses.  To see why,  consider the subjective relative clause, both  models will results in the same vector, depicted as follows:
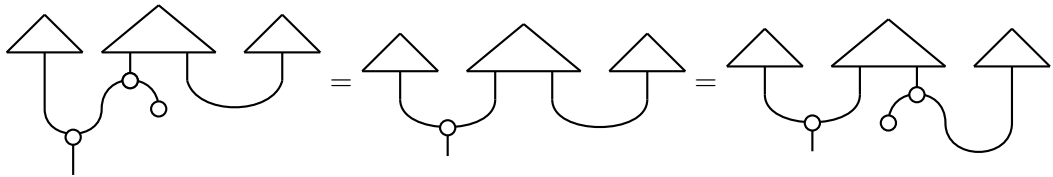
\[
\hspace{-1cm} {%
\beginpgfgraphicnamed{subj-rel-copy}
\InputIfFileExists{subj-rel-copy.tikz}{}{\input{./tikz/subj-rel-copy.tikz}}
\endpgfgraphicnamed}  =      {%
\beginpgfgraphicnamed{end-rel}
\InputIfFileExists{end-rel.tikz}{}{\input{./tikz/end-rel.tikz}}
\endpgfgraphicnamed} = {%
\beginpgfgraphicnamed{obj-rel-copy}
\InputIfFileExists{obj-rel-copy.tikz}{}{\input{./tikz/obj-rel-copy.tikz}}
\endpgfgraphicnamed}
\]

In the above models, one could substitute the concrete linear maps of  verbs into the normal forms of  relative clauses, as those models were formed based on a categorical setting and the type of the verb matched with the spaces required by our setting. This is not the case for the multiplicative model of \cite{Lapata},  since here the meaning vector of a string of words $\text{w}_1 \text{w}_2 \cdots \text{w}_n$  is the component wise multiplication of their context vectors $\ov{w}_1 \odot \ov{w}_2 \odot \cdots \odot \ov{w}_n$. For instance,  the meaning vector of a subjective  relative clause "Subj who Verb Obj" is the vector $\ov{\text{Subj}} \odot \ov{\text{who}} \odot \ov{\text{Verb}} \odot \ov{\text{Obj}}$. 

In the compact categorical setting, component wise multiplication of two vectors  can be modelled by the $\mu$ map. For  two vectors  $\ov{v} = \sum_i a_i \ov{x}_i$ and $\ov{w} = \sum_j b_j \ov{x}_j$ we have:
\[
 \mu(\ov{v}, \ov{w})  \ = \ \mu(\ov{v} \otimes \ov{w}) \ = \  \sum_i a_i b_i \ov{x}_i = \ov{v} \odot \ov{w}
\]
As a result we obtain  $\ov{v} \odot \ov{w} =  \ov{w} \odot \ov{v}$; diagrammatically we have:
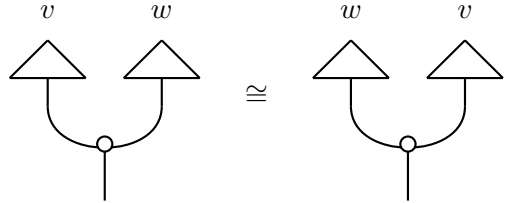
\begin{center}
  {%
\beginpgfgraphicnamed{Frob-Mult}
\InputIfFileExists{Frob-Mult.tikz}{}{\input{./tikz/Frob-Mult.tikz}}
\endpgfgraphicnamed}  \quad $\cong$ \quad   {%
\beginpgfgraphicnamed{Frob-Mult2}
\InputIfFileExists{Frob-Mult2.tikz}{}{\input{./tikz/Frob-Mult2.tikz}}
\endpgfgraphicnamed} 
\end{center}

One problem with this model is  the meaning of `who' , which is either ignored or taken to be its context vector, both of which  are problematic.  A more serious problem is that this model does not respect the word order, hence  cannot   distinguish between grammatical and ungrammatical strings, and even in the grammatical cases, it assigns equal meanings to strings with the same words but in different orders. For instance, the  clause `men who eat  rabbits' and `rabbits who eat men'  will have the same meaning, which    will be the same as the vector for the ungrammatical strings `rabbits men  ate'.

\section{A Toy  Experiment}

For demonstration purposes and  in order to get  an idea on how the data from a large scale corpus respond to the abstract computational methods developed here, we implement the  concrete instantiation of Section \ref{conc} on the British National Corpus (BNC) and  do a small-scale experiment, using parameters  of previous experimentations with the categorical model\cite{GrefenSadr1,GrefenSadr2,Kartetal1,Kartetal2}.  

The task we consider   is a term/description classification task, similar to the term/definition classification task of Kartsaklis et al (2012) \cite{Kartetal2}, where we replace    \emph{definitions}  with \emph{descriptions}.  The motivation behind this task is  the fact that  relative clauses are often used to describe or provide extra information  about words.    For instance,     the word  `football' may by described by the clause `game that boys like' and the word  `doll'  by the clause  `toy that girls prefer'.  For our  experiment,  we chose a set of words and manually described each of them by an appropriate relative clause. This data set is presented in the following table:

\begin{center}
\begin{tabular}{c|c|c}
&Term  & Description\\
\hline
	  1& emperor  & person who rule empire\\
Ê Ê Ê Ê 2& queen  & woman who reign country\\
Ê Ê Ê Ê 3& mammal  & animal which give birth\\
Ê Ê Ê Ê 4& plug  & plastic which stop water\\
Ê Ê Ê Ê 5& carnivore  & animal who eat meat\\
Ê Ê Ê Ê 6& vegetarian  & person who prefer vegetable\\
Ê Ê Ê Ê 7& doll  & toy that girl prefer\\
Ê Ê Ê Ê 8& football  & game that boy like\\
Ê Ê Ê Ê 9& skirt  & garment that woman wear
\end{tabular}
\end{center}

The goal of the task is to compute for what percentage of the words,   their  vectors are closest to the vectors of their descriptions.  For this purpose, we compared the vector of each term to the vectors of  all the descriptions. From our 9 terms, the cosines of 6 of them were the closest to their own description. The terms whose descriptions were not the closest to them were `plug, carnivore', and `vegetarian'. In all of these cases  the correct description was the second closest to the term. Below are  some  exemplary  term/description cosines. 

\medskip
\begin{center}\begin{tabular}{l|l||c}
\hline
Term & \quad Description & Cosine \\
\hline
football&game that boys like &0.62\\
   &woman who reigns country& 0.24\\
   &toy that girls prefer& 0.18\\
\hline
doll& toy that girls prefer&0.36\\
   &woman who reigns country &0.35\\
   &man who rules empire &0.30\\
\hline
mammal& animal which gives birth & 0.36\\
   &man who rules empire&0.19\\
   &woman who reigns country&0.17\\
\hline
ball& man who rules empire&0.35\\
   &woman who reigns country & 0.28\\
  & toy that girls prefer &0.22\\
\hline
plug&toy that girls prefer& 0.24\\
   &plastic which stops water& 0.22\\
   &woman who reigns country& 0.17\\
\hline

\end{tabular}
\end{center}

\bigskip
To have a comparison baseline, we also experimented with a multiplicative and an additive model.  In these models,  one builds  vectors   for the relative clauses by multiplying/summing the vectors of the words in them  in two ways:  (1) treating  the relative pronoun as noise and   not considering it in the computation, (2) treating the relative pronoun as any other word and considering its  context vector  in the computation. As prevue, the results were not to  be  significantly different (as the vectors of  relative pronouns were dense  and computing  by them was similar to  computing with the  vector consisting of all 1's).   For instance, the cosine between the vector of  `football' and its description was 0.39 when the pronoun was not considered and 0.37 when it was considered.  Both of these are lower than in our model, where this cosine was    0.61.   The first model got 6 out of the 9 cases correct; in the second model this number decreased to 5 out of 9.  The individual results were mixed, in that they did bad on  some of  the good results of the Frobenius model, for example in the multiplicative model, the closest description to the term `mammal' was `animal that eats meat'. At the same time,  they  performed better on some of the bad terms of the Frobenious model, for example the description of `plug' in the multiplicative model became  the closest  to it.   However, the above numbers are not real representatives of the performance of the models; as the dataset is small and hand-made. Extending this task to a  real one on a large automatically built  dataset and doing more involved statistical analysis on the results requires a different venue; it constitutes work-in-progress.

\section{Conclusion and Future Directions}

In this paper, we have extended the compact categorical semantics of
 Coecke et al. (2010) \cite{Coeckeetal} to analyse meanings of relative clauses in
English from a vector space point of view. The resulting vector space
semantics of the pronouns and clauses is based on the Frobenius
algebraic operations on vector spaces: they reveal the internal
structure, or what we call \emph{anatomy}, of the relative clauses. 

The methodology pursued in this paper and the Frobenius operations can
be used to provide semantics for other relative pronouns and also
other closed-class words such as prepositions and determiners. In each
case, the grammatical type of the word and a detailed analysis of the
role of these words in the meaning of the phrases in which they occur
would be needed. In some cases, it may be necessary to introduce a
linear map to represent the meaning of the word, for instance to
distinguish the preposition {\em on} from {\em in}.

The contribution of this paper can be better demonstrated by using the string
diagrammatic representations of the vector space meanings of these
clauses. A noun phrase modified by a subject relative clause, which
before this paper was depicted as follows:

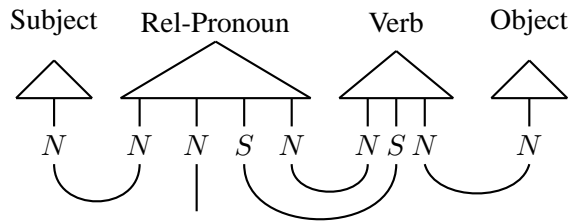
\begin{center}\begin{minipage}{7cm}
  {%
\beginpgfgraphicnamed{subj-sem-hidden}
\InputIfFileExists{subj-sem-hidden.tikz}{}{\input{./tikz/subj-sem-hidden.tikz}}
\endpgfgraphicnamed}\end{minipage}
\end{center}  
\noindent
will now include the internal anatomy of its relative pronoun:

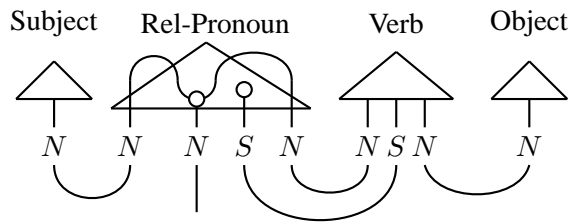
\begin{center}\begin{minipage}{7cm}
  {%
\beginpgfgraphicnamed{subj-sem-anatomy}
\InputIfFileExists{subj-sem-anatomy.tikz}{}{\input{./tikz/subj-sem-anatomy.tikz}}
\endpgfgraphicnamed}\end{minipage}
\end{center}  

\noindent This internal structure shows how the information from the
noun flows through the relative pronoun to the rest of the clause and
how it interacts with  other words. We have instantiated this
vector space semantics using truth-theoretic and corpus-based
examples.

One aspect of our example spaces which means that they work
particularly well is that the sentence dimension in the verb is
already discarded, which means that the $\iota$ maps are not required
(as discussed above). Another feature is that the simple nature of the
models means that the $\mu$ map does not lose any information, even
though it takes the diagonal of a matrix and hence in general throws
information away. The effect of the $\iota$ and $\mu$ maps in more
complex representations of the verb remains to be studied in future
work.

On the practical side, what we offer in this paper is a method for
building appropriate vector representations for relative clauses. As a
result, when presented with a relative clause, we are able to build a
vector for it, only by relying on the vector representations of the
words in the clause and the grammatical role of the relative
pronoun. We do not need to retrieve information from a corpus to be
able to build a vector or linear map for the relative pronoun, neither
will we end up having to discard the pronoun and ignore the role that
it plays in the meaning of the clause (which was perhaps the best
option available before this paper).
However, the Frobenius approach and our claim that the resulting
vectors are `appropriate' requires an empirical evaluation. Tasks
such as the term definition task from  Kartsaklis et al. (2010) \cite{Kartetal1} (which also
uses Frobenius algebras but for a different purpose) are an obvious
place to start. More generally, the sub-field of compositional
distributional semantics is a growing and active one
\cite{Mitchell:Lapata:08,BaroniEMNLP10,zanzotto:coling10,socher}, for
which we argue that high-level mathematical investigations such as
this paper, and also  Clarke (2008) \cite{clarke:thesis08}, can play a crucial
role.

\section{Acknowledgements}

We would like to thank Dimitri Kartsaklis and Laura Rimell for helpful
comments. Stephen Clark was supported by ERC Starting Grant DisCoTex
(30692). Bob Coecke and Stephen Clark are supported by EPSRC Grant EP/I037512/1. 
Mehrnoosh Sadrzadeh is supported by an EPSRC CAF EP/J002607/1.

\bibliography{JLC}
\bibliographystyle{splncs03}

\end{document}